\documentclass[10pt,psfig]{article}
\usepackage[left=2cm,right=2cm,top=2cm,bottom=2cm]{geometry}
\usepackage{dsfont}
\usepackage{amsmath}

 \def\spacingset#1{\renewcommand{\baselinestretch}%
{#1}\small\normalsize} \spacingset{1}

\usepackage{graphicx}
 \usepackage{amsmath}
 \usepackage{amssymb}
 \usepackage{amsthm} 	
\usepackage{color}
\usepackage{placeins}
\usepackage{ulem}

\usepackage{caption}
\usepackage{yhmath}

\usepackage{float}

\newtheorem{theorem}{Theorem}
\newtheorem{lemma}[theorem]{Lemma}

\usepackage{xcolor,colortbl}
\definecolor{Gray}{gray}{0.85}
\newcolumntype{g}{>{\columncolor{Gray}}c}
\newcolumntype{w}{>{\columncolor{white}}c}
\definecolor{LightCyan}{rgb}{0.88,1,1}

\newcommand{\tb}{\textcolor{black}}

\newcommand{\Z}{\mathbb{Z}}

\newcommand{\XX}{\mathbb{X}}
\newcommand{\PP}{\mathbb{P}}
\newcommand{\RR}{\mathbb{R}}

  \title{Kernel distance measures for time series, random fields and other structured data}
\author {Srinjoy Das \footnote{
School of Mathematical and Data Sciences, 
West Virginia University, 
Morgantown, WV 26506, USA, 
email: {\tt  srinjoy.das@mail.wvu.edu}}
\and
Hrushikesh N. Mhaskar \footnote{
Institute of Mathematical Sciences,
Claremont Graduate University,
Claremont, CA 91711, USA,
email: {\tt hrushikesh.mhaskar@cgu.edu}}
\and
Alexander Cloninger \footnote{
Department of Mathematics and 
Hagi{\u g}olou Data Science Institute,
   University of California---San Diego,
   La Jolla, CA 92093, USA
   email: {\tt acloninger@ucsd.edu}}
} 
\date{ } 

\begin{document}


\newcolumntype{g}{>{\columncolor{Gray}}c}
 \maketitle
 

\begin{abstract}
 \tb{This paper introduces {\bf kdiff}, a novel kernel-based measure for estimating distances between instances of time series, random fields and other forms of structured data. This measure is based on the idea of matching distributions that only overlap over a portion of their region of support.
Our proposed measure is inspired by {\bf MPdist} which has been previously proposed for such datasets and is constructed using Euclidean metrics, whereas {\bf kdiff} is constructed using non-linear kernel distances.
Also, 
{\bf kdiff} accounts for both self and cross similarities across the instances and is defined using a lower quantile of the distance distribution.  Comparing the cross similarity to self similarity allows for measures of similarity that are more robust to noise and partial occlusions of the relevant signals.
Our proposed measure {\bf kdiff} is a more general form of the well known kernel-based Maximum Mean Discrepancy ({\bf MMD}) distance estimated over the embeddings. Some theoretical results are provided for separability conditions using {\bf kdiff} as a distance measure for clustering and classification problems where the embedding distributions can be modeled as two component mixtures. Applications are demonstrated for clustering of synthetic and real-life time series and image data, and the performance of {\bf kdiff} is compared to competing distance measures for clustering.}
\end{abstract}
 

\section{Introduction and Motivation}

Clustering and classification tasks in applications such as time series and image processing are critically dependent on the distance measure used to identify similarities in the available data. In such contexts, several distance measures have been proposed in the literature:
\begin{itemize}
\item Point-to-point matching e.g. Euclidean distance or Dynamic Time Warping distance \cite{ratanamahatana2004everything, keogh2005exact}
\item \tb{Matching features of the series e.g. autocorrelation coefficients \cite{pierpalo2009autocorrelation}, Pearson correlation coefficients \cite{golay1998new}, periodograms \cite{alonso2014func}, extreme value behavior \cite{d2017fuzzy}}
\item Number of matching subsequences in the series \cite{gharghabi2018matrix}
\item Similarity of embedding distributions of the series \cite{brandmaier2011permutation}
\end{itemize}

 \tb{In this paper we consider distance measures for applications involving clustering, classification and related data mining tasks in time series, random fields and other forms of  possibly non i.i.d data.  In particular, we focus on problems where membership in a specific class is characterized by instances matching only over a portion of their region of support. In addition, the regions where such feature matching occurs may not be overlapping in time, or on the underlying grid of the random field. 
Distance measures must take these data characteristics into consideration when determining similarity in such applications. Previously {\bf MPdist} has been proposed as a distance measure for such time series datasets \cite{gharghabi2018matrix} which match only over part of their region of support and is constructed using Euclidean metrics. Inspired by {\bf MPdist}, we propose a new kernel-based distance measure {\bf kdiff} for estimating distances between instances of such univariate and multivariate time series, random field and other types of structured data.}


{For constructing {\bf kdiff}, we first create sliding window based embeddings over the given time series or random fields. We then estimate} a distance distribution by using a kernel-based distance measure between \tb{such embeddings over given pairs of data instances}. Finally the distance measure used in clustering, classification and related tasks is defined by a pre-specified lower quantile of this distance distribution. This kernel-based distance measure based on such embeddings can also be constructed using the Reproducing Kernel Hilbert Space (RKHS) based Maximum Mean Discrepancy (\textbf{MMD}) previously discussed in \cite{gretton2012kernel}.  Our kernel based \tb{measure {\bf kdiff} can be considered as a more general distance compared to RKHS MMD for applications where class instances match only over a part of the region of support}. More details about the connections \tb{between {\bf kdiff} and} RKHS MMD are provided later in the paper.  We also note that the kernel construction in \textbf{kdiff} allows for data-dependent kernel construction similar to \textbf{MMD} \cite{gretton2012optimal, jitkrittum2016interpretable, cheng2020two}, though we focus on isotropic localized kernels in this work and compare to standard \textbf{MMD}.

The rest of the paper is organized as follows. Section \ref{kdiff.mainidea} outlines the main idea and motivation behind the construction of our distance measure {\bf kdiff}. 
Section \ref{kdiff.mainidea} also outlines some theoretical results for separability of data \tb{using {\bf kdiff} as a distance measure} for clustering, classification and related tasks \tb{by modeling the embedding distributions derived from the original data as \tb{two component} mixtures.} 
Section \ref{kdiff.parameters} outlines some practical considerations and data-driven strategies to determine optimal parameters for the algorithm \tb{to estimate {\bf kdiff}.} Section \ref{kdiff.numwork} presets simulation results using {\bf kdiff} on both synthetic and real-life datasets and compares them with existing methods. Finally Section \ref{kdiff.conclusions} outlines some conclusions and directions for future work.



\section{Main Idea}
\label{kdiff.mainidea}

\subsection{Overview}
\label{kdiff.ss.overview}

Consider \tb{two} real-valued datasets $\{{\bf X}_{\underline t}, {\bf Y}_{\underline t} : \underline t \in \Z^k\}$ defined over a $k$-dimensional index set. These may in general be vector-valued random variables, and therefore ${\bf X}_{\underline t}$ and ${\bf Y}_{\underline t}$ can be considered as \tb{either univariate or multivariate time series, random fields or other types of structured data.}  Our problem of interest is where \tb{instances of } ${\bf X}_{\underline t}$ and ${\bf Y}_{\underline t}$ match with certain localized motifs $\{{\bf X}_{\underline t}: t\in  S\} \approx \{{\bf Y}_{\underline t}: t\in  S'\}$ for small localized index sets $S,S' \subset \Z^k$.
For both the univariate and multivariate cases, we can embed these data sets into some corresponding point clouds ${\bf X},{\bf Y} \subset \mathbb{R}^L$ via windowing with a size $L$ window, where $L$ can be determined from training or some other appropriate technique \cite{brandmaier2011permutation, bandt2002permutation}.  Once we have a window embedding of these data sets, we can define various distance measures on the resulting point clouds to define similarity between ${\bf X}_{\underline t}$ and ${\bf Y}_{\underline t}$. 

 A distance measure that has been proposed previously to determine similarity between two such time series embedded point clouds constructed over $\mathbb{R}^L$ is {\bf MPdist}  \cite{gharghabi2018matrix}. In this case, a cross-data distance measure, denoted $D^2$, can be constructed by using 1-nearest neighbor Euclidean distances between point clouds ${\bf X}$ and ${\bf Y}$ as below:

$$
d({x}) = \inf_{{y} \in {\bf Y}} || {x} - { y} ||, \ \  \forall {x} \in {\bf X}
$$

$$
d({y}) = \inf_{{x} \in {\bf X}} || {x} - {y} ||, \ \  \forall {y} \in {\bf Y}
$$

\begin{equation}
D^2 = \{d^2({z}) : z\in {\bf X}\cup {\bf Y}\}.
\label{eq.mpdist_crosssim}
\end{equation}

 \tb{In \cite{gharghabi2018matrix}, the distance measure {\bf MPdist} was estimated for univariate time series by choosing the $k^{th}$ smallest element in the set $D^2$.
In general, {\bf MPdist} can be constructed using a lower quantile of the distance distribution $D^2$.}

 \tb{Our proposed distance measure {\bf kdiff} generalizes {\bf MPdist} using a kernel-based construction, and by considering both cross-similarity and self-similarity.  Similar to {\bf MPdist}, we first construct sliding window based embeddings over the original data instances ${\bf X}_{\underline t}, \ {\bf Y}_{\underline t}$ and obtain corresponding point clouds ${\bf X}$ and ${\bf Y}$. For {\bf MPdist} the final distance is estimated based on cross-similarity between the embeddings ${\bf X},\ {\bf Y}$ as shown in Equation \ref{eq.mpdist_crosssim}. Our distance measure {\bf kdiff} differs from this in two ways:
\begin{itemize}
\item We use a kernel based similarity measure over the obtained sliding window based embeddings ${\bf X}$ and ${\bf Y}$ for {\bf kdiff}  instead of the Euclidean metric used in {\bf MPdist}.
\item For {\bf kdiff} the final distance is estimated based on both self and cross-similarities between the embeddings ${\bf X},\ {\bf Y}$ respectively. \tb{The inclusion of self-similarity in the construction of {\bf kdiff} as compared to only cross-similarity for {\bf MPdist} leads to better clustering performance for  data with reduced signal-to-noise ratio of the matching region versus the background. This is demonstrated empirically for both synthetic and real-life data in Section \ref{kdiff.numwork}.} 
\end{itemize}}

\subsection{The construction of \textbf{kdiff}}\label{bhag:kdiffconst}

To define our \textbf{kdff} statistic, we will begin with a discussion of general distributions defined on $\mathbb{R}^L$.  For the purposes of this paper, these can be assumed to be the distributions that the finite samples ${\bf X}$ are drawn from (in a non-iid fashion) and stitched together to form the time series ${\bf X}_{\underline t}$ (respectively for ${\bf Y}$ and ${\bf Y}_{\underline t}$).  

In general, we can define the distributions on be defined on $\XX$, which is a locally compact metric measure space with the metric $\rho$ and a distinguished probability measure $\nu^*$.   The term measure will denote a signed (or  positive with bounded total variation) Borel measure. 
We introduce a fixed, positive definite kernel $K:\XX\times\XX\to (0,\infty)$, $K\in C_0(\XX\times\XX)$. 
Since the kernel is fixed, the mention of this kernel will be omitted from notations, although the kernel plays a crucial role in our theory.  
Given any signed  measure (or positive measure having bounded variation) $\tau$ on $\XX$, we define the
\textbf{witness function} of $\tau$ by
\begin{equation}
\label{eq:potentialdef}
U(\tau)(z)=\int_\XX K(z,x)d\tau(x), \qquad z\in \XX,
\end{equation}
and similarly the magnitude of the witness function
\begin{equation}
\label{eq:witnessdef}
T(\tau)(z)=|U(\tau)(z)|, \qquad z\in\XX.
\end{equation}
In the context of defining a distance between $\mu_1$, $\mu_2$, we take $\tau = \mu_1-\mu_2$, which results in a witness function 
\begin{align*}
U(\mu_1-\mu_2)(z) = \mathbb{E}_{x\sim \mu_1} [K(z,x)] - \mathbb{E}_{y\sim \mu_2} [K(z,y)]. 
\end{align*}

To quantify where $T(\mu_1-\mu_2)$ is small, we define the cumulative distribution function (CDF) of a Borel measurable function $f:\XX\to\RR$ by
\begin{equation}
\label{eq:CDFdef}
\Lambda(f)(t)=\Lambda(\nu^*;f)(t)=\nu^*\left(\{z: |f(z)|< t\}\right), \qquad t\in [0,\infty),
\end{equation}
and its ``inverse'' CDF by
\begin{equation}
\label{eq:CDFinv}
f^\#(u)=\sup\{t \in \RR: \Lambda(f)(t)\le u\}, \qquad u\in [0,\infty).
\end{equation}
Both $\Lambda(f)$ and $f^\# $ is a non-decreasing functions, and $f^\#(u)$ defines the \textbf{$u$-th quartile} of $f$. 

Finally, we are prepared to define our \textbf{kdiff} distance between probability measures $\mu_1$, $\mu_2$.   Having defined $T(\mu_1-\mu_2)(z)$, we now define \textbf{kdiff} to be the $\alpha$ quantile of $T(\mu_1-\mu_2)$,
\begin{equation}
\label{eq:kdiffdef}
\mathbf{kdiff}(\mu_1,\mu_2;\alpha)=(T(\mu_1-\mu_2))^\#(\alpha), \qquad \alpha\in (0,1).
\end{equation}

The intuition of \eqref{eq:kdiffdef} is that, if $\mu_1=\mu_2$, the resulting {\bf kdiff} statistic will be zero.  But beyond this, if $T(\mu_1-\mu_2)(z)=0$ for a set $z\in A\subset \XX$ such that $\nu^*(A)>0$, then for a localized enough kernel, there exists a quantile $\alpha$ for which we can still have the resulting {\bf kdiff} statistic be close to zero.  This allows us to match distributions that agree over partial support.  This will be discussed more precisely in Section \ref{kdiff.separability}.

\subsection{Separability Theorems for \textbf{kdiff}}\label{kdiff.separability}

For the purposes of analyzing the \textbf{kdiff} statistic, we will focus on the setting of resolving mixture models of probabilities on $\XX$ when only one of the components agree. 
Accordingly, for any $\delta\in (0,1)$, we define $\PP_\delta$ to be the class of all probability measures $\mu$ on $\XX$ which can be expressed as $\mu=\delta \mu_F+(1-\delta)\mu_B$, where $\mu_F$ and $\mu_B$ are probability measures on $\XX$.
With the applications in the paper in mind, we will refer to $\mu_F$ as the \textit{foreground} and $\mu_B$ as the \textit{background} probabilities.
Our interest is in developing a test to see whether given two measures $\mu_1$ and $\mu_2$ in $\PP_\delta$, the corresponding foreground components agree. 
Clearly, the same discussion could also apply to the case when we wish to focus on the background components with obvious changes.  

We first present some preparatory material before reaching our desired statements.
For any subset $A\subseteq \XX$ and $x\in\XX$, we define
\begin{equation}
\label{eq:distancefromset}
\mathsf{dist}(A,x)=\inf_{y\in A}\rho(y,x).
\end{equation}

The support of a finite positive measure $\mu$, denoted by $\mathsf{supp}(\mu)$  is the set of all $x\in\XX$ such that $\mu(U)>0$ for all open subsets $U$ containing $x$. 
Clearly, $\mathsf{supp}(\mu)$ is a closed set. 
If $\sigma$ is a non-zero signed measure and $\sigma=\sigma^+-\sigma^-$ is the Jordan decomposition of $\sigma$ then we define $\mathsf{supp}(\sigma)=\mathsf{supp}(\sigma^+)\cup \mathsf{supp}(\sigma^-)$. 
If $f :\XX\to\RR$, we define
$$
\|f\|_\infty=\sup_{x\in\XX}|f(x)|.
$$

The following lemma summarizes some important but easy properties of quantities $\Lambda(f)$ and $f^\#$ defined in \eqref{eq:CDFdef} and \eqref{eq:CDFinv} respectively.
\begin{lemma}
\label{lemma:rearrangement}
{\rm (a)} For $t, u\in [0,\infty)$,
\begin{equation}
\label{eq:inversefn}
\Lambda(f)(t)\le u \Rightarrow t\le f^\#(u), \qquad u<\Lambda(f)(t) \Rightarrow  f^\#(u) \le t.
\end{equation}
{\rm (b)} If $\epsilon>0$, $f, g :\XX\to\RR$, and $\|f-g\|_\infty \le\epsilon$, then $\sup_{u\in\RR}|f^\#(u)-g^\#(u)| \le \epsilon$.
\end{lemma}

Our goal is to investigate sufficient conditions on two measures in $\PP_\delta$ so that $\mathbf{kdiff}$ can distinguish if the foreground components of the measures are the same. 
For this purpose, we introduce some further notation, where we suppress the mention of certain quantities for brevity.
Let $\mu_j=\delta\mu_{j,F}+(1-\delta)\mu_{j,B}\in \PP_\delta$, $j=1,2$, and $\mathbb{S}_F=\mathsf{supp}(\mu_{1,F})\cup \mathsf{supp}(\mu_{2,F})$, and we define $\mathcal{S}^c = \mathbb{X}\setminus \mathcal{S}$.  
We define for $\eta,\theta>0$,
\begin{equation}
\label{eq:setdef}
\begin{aligned}
\mathcal{S}_B(\mu_1,\mu_2;\eta)&=\{z\in \mathbb{X} : T(\mu_{1,B}-\mu_{2,B})(z) <\eta\}, & \phi_B(\eta)&=\nu^*(\mathcal{S}_B(\mu_1,\mu_2;\eta)) \\
\mathcal{S}_F(\mu_1,\mu_2;\eta)&=\{z\in \mathbb{X} : T(\mu_{1,F}-\mu_{2,F})(z) <\eta\}, & \phi_F(\eta)&=\nu^*(\mathcal{S}_F(\mu_1,\mu_2;\eta)) \\
\mathcal{G}(\mu_1,\mu_2;\theta,\eta)&= \left(\mathcal{S}^c_F(\theta)\cap \mathcal{S}_B(\eta)\right) \cup \left(\mathcal{S}^c_B(\theta)\cap \mathcal{S}_F(\eta)\right), & \psi(\theta,\eta)&=\nu^*(\mathcal{G}(\mu_1,\mu_2;\theta,\eta))
\end{aligned}
\end{equation}

\begin{theorem}
\label{theo:kdiff}
Let $\delta\in (0,\frac{1}{2})$, $\mu_j=\delta\mu_{j,F}+(1-\delta)\mu_{j,B}\in \PP_\delta$ ($j=1,2$). \\[1ex]
{\rm (a)} 
If  $\eta>0$ and $\mu_{1,F}=\mu_{2,F}$ then for any $\alpha\le \phi_B(\eta)$, we have $\mathbf{kdiff}(\mu_1,\mu_2;\alpha)\le (1-\delta)\eta$.\\[1ex]
{\rm (b)} 
If $\eta>0$ such that $\phi_F\left(\frac{3(1-\delta)}{\delta} \eta\right)<1$ and $\psi\left(\frac{3(1-\delta)}{\delta} \eta,\eta\right)>0$, 
then $\mu_{1,F}\not=\mu_{2,F}$ and for any $\alpha$ with $$1- \psi\left(\frac{3(1-\delta)}{\delta}\eta, \eta\right)\le \alpha,$$ we have $\mathbf{kdiff}(\mu_1,\mu_2;\alpha)\ge 2(1-\delta)\eta$.
\end{theorem}

\begin{proof}\ 
To prove part (a), we observe that since $\mu_{1,F}=\mu_{2,F}$, $T(\mu_1-\mu_2)(z) = T(\mu_{1,B}-\mu_{2,B})(z)$ for all $z\in\XX$. By definition \eqref{eq:setdef},
$$
\mathcal{S}_B(\mu_1,\mu_2;\eta) = \{z\in\XX : T(\mu_1-\mu_2)(z) < (1-\delta) \eta\}.
$$
Therefore, $\phi_B(\eta)\le \Lambda(T(\mu_1-\mu_2))((1-\delta)\eta)$. 
In view of \eqref{eq:inversefn}, this proves part (a). 

To prove part (b), we will write
$$
\theta=\frac{3(1-\delta)}{\delta}\eta.
$$
Our hypothesis that $\phi_F(\theta)<1$ means that 
 $\mu_{1,F}\not=\mu_{2,F}$ and $\mathcal{S}^c_F(\theta)$ is nonempty.
For all $z\in \mathcal{S}^c_F(\theta)\cap \mathcal{S}_B(\eta)$,
$$
T(\mu_1-\mu_2)(z)\ge \delta |U(\mu_{1,F}-\mu_{2,F})(z)|- (1-\delta)|U(\mu_{1,B}-\mu_{2,B})(z)|>\delta\theta-(1-\delta)\eta\ge 2(1-\delta)\eta.
$$
Moreover, for $z\in \mathcal{S}^c_B(\theta)\cap \mathcal{S}_F(\eta)$, we also know that 
$$
T(\mu_1-\mu_2)(z) \ge (1-\delta)|U(\mu_{1,B}-\mu_{2,B})(z)| - \delta|U(\mu_{1,F}-\mu_{2,F})(z)| \ge (1-\delta)\theta - \delta\eta \ge \frac{3(1-\delta)^2-\delta^2}{\delta}\eta.
$$
Note that because $\delta<\frac{1}{2}$, we have $\frac{3(1-\delta)^2-\delta^2}{\delta} > 2(1-\delta)$.  So, this means that
$$
\{z\in\mathbb{X} : T(\mu_1-\mu_2)(z)< 2(1-\delta)\eta\}\subset \{z\in\mathbb{X}  :   z\not\in \mathcal{G}(\theta,\eta)\};
$$
This means that $\Lambda(T(\mu_1-\mu_2))(2(1-\delta)\eta)\le 1-\psi(\theta,\eta)$. 
Since $\alpha\ge  1-\psi(\theta,\eta)$, this estimate together with \eqref{eq:inversefn} leads to the conclusion in part (b).
\end{proof}

We wish to comment on the practicality of the constants $\mathcal{S}_F(\mu_1,\mu_2;\eta)$, $\mathcal{S}_B(\mu_1,\mu_2;\eta)$ and $\mathcal{G}(\mu_1,\mu_2;\theta,\eta)$.   We consider this with the simple setting where $K$ is a compactly supported localized kernel (e.g., indicator function of an $\epsilon$-ball) in order to avoid the discussion of tails.  We define the well-separated setting as the setting where $\rho(\mu_{1,B},\mu_{2,B})>\epsilon$ and $\rho(\mu_{1,B},\mu_{i,F})>\epsilon$.  For part (b), we'll also use $\rho(\mu_{1,F},\mu_{2,F})>\epsilon$ and all four measures are sufficiently concentrated, i.e., $\mu_{i,F}\left(\{z\in \mathbb{X} : T(\mu_{i,F})(z)\ge \theta\}\right)\ge 1-\xi$ and $\mu_{i,B}\left(\{z\in \mathbb{X} : T(\mu_{i,B})(z)\ge \theta\}\right)\ge 1-\xi$. 
We consider the results of Theorem \ref{theo:kdiff} in the well-separated setting with $\nu^*=\frac{1}{2}(\mu_1+\mu_2)$:
\begin{enumerate}
\item[(a)] $\mathcal{S}_B(\mu_1,\mu_2;\eta)$ measures how much the backgrounds overlap with one another.  In this setting,  $\phi_B(\eta)\ge\delta$ for any $\eta> 0$.  This is because $T(\mu_{1,B}-\mu_{2,B})(z)=0$ for all $z\in \mathsf{supp}(\mu_{i,F})$, and thus $z\in \mathcal{S}_B(\mu_1,\mu_2;\eta)$.  Since $\nu^*( \mathsf{supp}(\mu_{1,F})\cup \mathsf{supp}(\mu_{2,F}))=\delta$, this lower bounds $\phi_B(\eta)$.  This means for any $\alpha<\delta$, $\textbf{kdiff}(\mu_1,\mu_2;\alpha)=0$.
\item[(b)]
Because of the well-separated assumption, $\mathcal{S}^c_F(\theta) \subset \mathcal{S}_B(\eta)$.  This means that everywhere the foregrounds are sufficiently concentrated, the backgrounds must be sufficiently small.    Similarly, $\mathcal{S}^c_B(\theta) \subset \mathcal{S}_F(\eta)$.  Furthermore, the sets $\mathcal{S}^c_F(\theta)\cap \mathcal{S}_B(\eta)$ and $\mathcal{S}^c_B(\theta)\cap \mathcal{S}_F(\eta)$ by definition are disjoint when $\theta>\eta$.
  Thus we have
\begin{align*}
1-\psi(\theta,\eta) &= 1- \nu^*\left((\mathcal{S}^c_F(\theta)\cap \mathcal{S}_B(\eta) ) \cup (\mathcal{S}^c_B(\theta)\cap \mathcal{S}_F(\eta))\right)\\
&= 1 -  \nu^*(\mathcal{S}^c_F(\theta)\cap \mathcal{S}_B(\eta) ) - \nu^*(\mathcal{S}^c_B(\theta)\cap \mathcal{S}_F(\eta)) \\
&= 1 -  \nu^*(\mathcal{S}^c_F(\theta)) - \nu^*(\mathcal{S}^c_B(\theta)) \\
&\le 1 - \delta(1-\xi) - (1-\delta)(1-\xi) \\
&= \xi.
\end{align*}
Thus we can choose $\eta,\theta$ as large as possible to satisfy the assumptions, and even then for very small quantiles $\alpha> \xi$ we $\textbf{kdiff}(\mu_1,\mu_2;\alpha) > 2(1-\delta)\eta$.
\end{enumerate}
\noindent These above descriptions clarify the theorem in the simplest setting.  When the foreground distributions are small but concentrated, and far from the separate backgrounds, then the hypothesis of $\mu_{1,F}\stackrel{?}{=}\mu_{2,F}$ can be easily distinguished with $\textbf{kdiff}$ for almost all $\alpha<\delta$.

In practice, of course, we need to estimate $\mathbf{kidff}(\mu_1,\mu_2;\alpha)$ from samples taken from $\mu_1$ and $\mu_2$. 
In turn, this necessitates an estimation of the witness function of probability measure from  samples from this probability.
We need to do this separately for $\mu_1$ and $\mu_2$, but it is convenient to formulate the result for a generic probability measure $\mu$.
To estimate the error in the resulting approximation, we need to stipulate some further conditions enumerated below.
We will denote by $\mathbb{S}^*=\mathsf{supp}(\mu)$.
\begin{description}
\item[Essential compact support] For any $t>0$, there exists $R(t)>0$ such that
\begin{equation}
\label{eq:esscompact}
K(x,y)\le t, \qquad x, y\in\XX,\ \rho(x,y)\ge R(t).
\end{equation}
\item[Covering property] For $t>0$, let $\mathbb{B}(\mathbb{S}^*, t)=\{z\in \XX: \mathsf{dist}(\mathbb{S}^*, z)\le R(t)\}$. There exist $A,\beta>0$ such that for any $t>0$, the set $\mathbb{B}(\mathbb{S}, t)$ is contained in the union of at most $At^{-\beta}$ balls of radius $\le t$. 
\item[Lipschitz condition]
We have
\begin{equation}
\label{eq:lipcond}
\max_{(x,y)\in \XX\times \XX}K(x,y)+\max_{x,x',y,y'\in \XX}\left\{\frac{|K(x,y)-K(x',y')|}{\rho(x,x')+\rho(y,y')}\right\}\le 1.
\end{equation}
\end{description}
Then H\"offding's inequality leads to the following theorem.  
\begin{theorem}
\label{theo:hoeffdingtheo}
Let $\epsilon>0$, $M\ge 2$ be an integer, and $\mu$ be any probability measure on $\XX$ and $\{y_1,\cdots, y_M\}$ be i.i.d. samples from $\mu$. 
Then with $\mu$-probability $\ge 1-\epsilon$, we have
\begin{equation}
\label{eq:hoeffding}
\left\|U(\mu)(\circ)-\frac{1}{M}\sum_{j=1}^M K(\circ, y_j)\right\|_\infty \le 2\left\{\frac{\log(4^{\beta+1}A/\epsilon)}{M}\right\}^{1/2}.
\end{equation}
\end{theorem}
The proof of Theorem \ref{theo:hoeffdingtheo} mirrors the results for the witness function in \cite{mhaskar2020witness}.

\subsection{\tb{Conclusions from Separability Theorems}}

To illustrate the benefit of the above theory, we recall the MMD distance measure between two probability measures $\mu_1$ and $\mu_2$ defined by
\begin{equation}
\label{eq:MMDdef}
\mbox{MMD}^2(\mu_1,\mu_2)=\int_\XX\int_\XX K(x,y)(d\mu_1(x)-d\mu_2(x))(d\mu_1(y)-d\mu_2(y)).
\end{equation}
When $\mu_1,\mu_2\in \PP_\delta$ and the foreground components $\mu_{1,F}=\mu_{2,F}$ then $\mu_1-\mu_2=(1-\delta)(\mu_{1,B}-\mu_{2,B})$ and
\begin{equation}
\label{eq:MMDdefect}
\mbox{MMD}^2(\mu_1,\mu_2)=(1-\delta)^2\mbox{MMD}^2(\mu_{1,B},\mu_{2,B}).
\end{equation}
Since $K$ is a positive definite kernel, it is thus impossible for $\mbox{MMD}^2(\mu_1,\mu_2)=0$ unless $\mu_{1,B}=\mu_{2,B}$. 
One of the  motivations for our construction is to devise a test statistic that can be arbitrarily small even if $\mu_{1,B}\not=\mu_{2,B}$.

The results derived above provide certain insights regarding when it is possible to perform tasks such as clustering and classification of data using distance measures such as {\bf kdiff} and {\bf MMD} based on the characteristics of their foreground and background distributions. The results in Theorem~\ref{theo:kdiff} show that, provided the backgrounds are sufficiently separated, the \textbf{kdiff} statistic will be significantly smaller when $\mu_{1,F} = \mu_{2,F}$ than when $\mu_{1,F}$ and $\mu_{2,F}$ are separated.

This enables {\bf kdiff} to perform accurate discrimination i.e. data belonging to the same class will be clustered correctly in this case. On the other hand, it is clear that even if $\mu_{1,F}=\mu_{2,F}$, $\mbox{MMD}^2(\mu_1,\mu_2)$ will still be highly dependent on the backgrounds.  In this paper we consider the case where data instances belonging to the same class have the same foreground but different background distributions. In such situations using synthetic and real life examples we demonstrate the comparative performance and effectiveness of {\bf kdiff} for clustering tasks versus other distance measures including $\mbox{MMD}$.


As a final note, we wish to mention the relationship between \textbf{MMD} and \textbf{kdiff}.   It can be shown that $\textbf{MMD}^2$ is the mean of the witness function with respect to $\nu^*=\frac{1}{2}(\mu_1+\mu_2)$, 
$\textbf{MMD}^2(\mu_1,\mu_2) = \mathbb{E}_{z\sim\nu^*}\left(|U(\mu_1-\mu_2)(z)|^2\right)$ \cite{jitkrittum2016interpretable, cloninger2018bounding}.
This is compared to our results for \textbf{kdiff}, or in particular $\textbf{kdiff}^2$.
Note that computing $\textbf{kdiff}(\mu_1,\mu_2;\alpha)^2$ is equivalent to computing $\textbf{kdiff}$ on the square of the witness function $T(\mu_1-\mu_2)(z) = |U(\mu_1-\mu_2)(z)|^2$, since quantiles depend only on the ordering of the underlying function.  This means the statistic $\textbf{kdiff}^2$ is simply taking the quantile of the square of the witness function, rather than the mean as in $\textbf{MMD}^2$.

\section{Estimation of Algorithm Parameters}
\label{kdiff.parameters}

The following parameters are required for estimation of the distance measure {\bf kdiff}:

\begin{itemize}
\item Length of sliding window {\bf $SL$} used to generate subsequences over \tb{given data} ({\bf embedding dimension}) 
\item Kernel bandwidth ({\bf $\sigma$}) of the Gaussian kernel $k(x,y)=e^{-\|x-y\|^2/2\sigma^2}$
\item Lower quantile $\alpha$ of the kernel-based distance distribution $T(z)$
\end{itemize}

\paragraph{Determining $SL$:}
  \tb{In this paper we demonstrate the application of {\bf kdiff} for clustering time series and random fields.} The sliding window length $SL$ is used to create subsequences \tb{(i.e. sliding window based embeddings)} over \tb{such} time series or random fields over which {\bf kdiff} is estimated. The number of subsequences formed depend on $SL$, the number of points in the time series or random field and the dimensionality of the data under consideration. Some examples are given as below:
\begin{itemize}
\item In case of a univariate time series of length $n$ \tb{if each subsequence is of length $L = SL$ then} there are $m = n - SL + 1$ embeddings
\item For a two dimensional $n$ x  $n$ random field \tb{if each subsequence has dimension $L = SL*SL$ then} there are $m = (n - SL + 1)^2$ embeddings
\item For a p-variate time series \tb{if each subsequence is of length $L = p * SL$ then} there are $m = n - L + 1$ embeddings
\end{itemize}

 The distance measure {\bf kdiff} is estimated over \tb{these} $m$ points in the $L$ dimensional embedding space. It is necessary to determine an optimal value of $SL$ to obtain accurate values of {\bf kdiff}. \tb{Very small} values of $SL$ may result in erroneous identification of the region where the time series or random field under consideration match. For example embeddings obtained in this manner may result in two dissimilar time series containing noise related fluctuations over a small region identified as "matching". On the other hand \tb{very high} values of $SL$ can lead to erroneous estimation of the distance distribution owing to less number of subsequences or sub-regions which results in incorrect estimates for {\bf kdiff}. \tb{As an optimal tradeoff between these competing considerations we determine the value of $SL$ based on the best clustering performance over a training set selected from the original data.}

\paragraph{Determining $\sigma$:}  Since {\bf kdiff} is a kernel-based similarity measure determination of the kernel bandwidth $\sigma$ is critical to the accuracy of estimation. In this case \tb{very small bandwidths for sliding window based embeddings ${\bf X}$ and ${\bf Y}$ derived from} two \tb{corresponding} time series or random fields can lead to incorrect estimates since only points in the immediate neighborhood of \tb{embeddings ${\bf X}$ and ${\bf Y}$} are considered in the estimation of the \textbf{kdiff} statistic. On the other hand very large bandwidths are also problematic since in this case any point ${\bf Z}$ becomes nearly equidistant from ${\bf X}$ and ${\bf Y}$ (here all points are considered in the embedding space), thereby causing the distance measure to lose sensitivity. \tb{To achieve a suitable tradeoff between these extremes we select $\sigma$ over a range of values of order equal to the k nearest neighbor distance over all points in the embedding space of ${\bf Z} = {\bf X} \cup {\bf Y}$ for a suitably chosen value of $k$. The optimal value of $\sigma$ is selected from this range based on the best clustering performance over a training set selected from the original data.} 

\paragraph{Determining $\alpha$:} The distance measure {\bf kdiff} is based on a lower quantile $\alpha$ of the estimated distance distribution over the embedding spaces of two time series or random fields. \tb{This quantile can be specified as a fraction of the total number of points in the distance distribution using either of the following methods below:}
\begin{itemize}
\item \tb{Based on exploratory data analysis, visual inspection or other methods if the extent of the matching portions of the time series, random fields or other data under investigation can be determined then $\alpha$ can be set as a fraction of the length or area of this matching region versus the overall span of the data}
\item \tb{For high dimensional time series or random fields $\alpha$ can be determined from a range of values based on the best clustering performance over a training set selected from the original data}
\end{itemize}


\section{Numerical work: simulations and real data}
\label{kdiff.numwork}

The effectiveness of the our novel \tb{distance measure} {\bf kdiff} for comparing two sets of data which match only partially over their region of support is estimated using \tb{kmedoids} clustering \cite{kaufmann1987clustering}. \tb{The kmedoids algorithm is a similar partitional clustering algorithm as kmeans which works by minimizing the distance between the points belonging to a cluster and the center of the cluster. However kmeans can work only with Euclidean distances or a distance measure which can be directly expressed in terms of Euclidean for example the cosine distance. In contrast the kmedoids algorithm can work with non Euclidean distance measures such as {\bf kdiff} and is also advantageous because the obtained cluster centers belong to one of the input data points thereby leading to greater interpretability of the results. For these reasons in this paper we consider kmedoids clustering with} $k=2$ classes and measure the accuracy of clustering for \tb{distance measures} {\bf kdiff}, mmd \cite{gretton2012kernel}, {\bf MPdist} \cite{gharghabi2018matrix} and dtw \cite{ratanamahatana2004everything, keogh2005exact} over synthetic and real time series and random field datasets as described in the following sections. \tb{Suitably chosen combinations of the parameters can be specified as described in Section \ref{kdiff.parameters} and the derived optimal values can then be used for measuring clustering performance with the test data using {\bf kdiff}}. Similar to {\bf kdiff}, \tb{distances measures} using Maximum Mean Discrepancy (mmd) and {\bf MPdist} are computed by first creating subsequences over the original time series or random fields. In both these cases the length of the sliding window $SL$ is determined based on the best clustering performance over a training set selected from the original data. Additionally for mmd which is also a kernel-based measure we determine the optimal kernel bandwidth ($\sigma$) based on training.

In this work we consider \tb{two synthetic and one real-life} datasets for measuring clustering performance with \tb{four distance measures} {\bf kdiff}, mmd, {\bf MPdist} and dtwd (Dynamic Time Warping distance).  For the synthetic datasets we generate the foregrounds and backgrounds as described in Section \ref{kdiff.separability} using autoregressive models of order $p$, denoted as AR($p$). These are models for a time series $W_t$ generated by 
$$
W_t = \sum_{i=1}^{p} \phi_i W_{t-i} + \epsilon_t, \qquad t = p+1, \ldots.
$$
where $\phi_1, \ldots, \phi_p$ are the $p$ coefficients of the AR($p$) model and $\epsilon_t$ can be i.i.d. Gaussian errors. We perform 50 \tb{Monte Carlo} runs over \tb{each} dataset and in each run we randomly divide the data into training and test sets. For each set of training data we determine the optimal values of the algorithm parameters based on the best clustering performance. \tb{Following this} we use these parameter values on the test data in each of the 50 runs. The final performance metric for a given distance measure is given by the total number of clustering errors for the test data over all 50 runs. \tb{The dtwclust package} \cite{sarda2017comparing} \tb{of R 3.6.2 has been used for implementation of the \tb{kmedoids} clustering algorithm and to evaluate the results of clustering}.

As a techincal note, as \textbf{MPdist} and \textbf{MMD} are generally computed as squared distances, we similarly work with $\textbf{kdiff}(\mu_1,\mu_2;\alpha)^2$ as the distance between distributions.  This is solely to ensure that the distances are based of Euclidean or kernel distances squares, and to ensure a fair comparison being fed into the kmedoids clustering algorithm.  Also as mentioned previously, computing $\textbf{kdiff}(\mu_1,\mu_2;\alpha)^2$ is equivalent to computing $\textbf{kdiff}$ on the square of the witness function $ |U(\mu_1-\mu_2)(z)|^2$, since quantiles depend only on the ordering of the underlying function.

\subsection{Simulation: Matching sub-regions in Univariate Time Series}
\label{kdiff.sim.univ_ts}

Data $Y_i$ for $i = 1,2,\ldots, 1000$ are simulated using the model (\ref{eq.ts_1d}). \tb{To generate this} the series $W_i$ are constructed via an AR(5) model driven by i.i.d errors $\sim N(0, 1)$. The AR(5) coefficients are set to $0.5, 0.1, 0.1, 0.1, 0.1$. 

\begin{equation}
Y_i = \mu + W_i 
\label{eq.ts_1d}
\end{equation}

 \tb{Following this} we form a {\it background} dataset  $X_{B_{j}}$  by generating $j = 1,2, \ldots, 21$ realizations of this data where the mean $\mu_j$ for realization $j$ is set as below:

\[
    \mu_j= 
\begin{cases}
    100*j & \text{if } j \geq 1 \ \text{and } \  j \leq 10\\
    100*(10-j) & \text{if } j \geq 11 \ \text{and } \  j \leq 20\\
    0              & \text{j=21}
\end{cases}
\]

 Next we generate a dataset $X_F$ consisting of 2 {\it foregrounds} $X_{FA}$ and $X_{FB}$ which enable forming the 2 classes to be considered for k-medoids clustering as follows.  For foreground $X_{FA}$ data $Y_i$ for $i = 1,2,\ldots, 50$ are simulated using the model (\ref{eq.ts_1d}). The series $W_i$ is constructed via an AR(1) model driven by i.i.d errors $\sim N(0, 1)$. The AR(1) coefficient is set to $0.1$ and $\mu=10$. For foreground $X_{FB}$ data $Y_i$ for $i = 1,2,\ldots, 25$ are simulated using the model (\ref{eq.ts_1d}). The series $W_i$ is constructed via an AR(1) model driven by i.i.d errors $\sim N(0, 1)$. The AR(1) coefficient is set to $0.1$ and $\mu=-10$. We \tb{then} form the {\it foreground} dataset $X_{F_{j}}$ by generating $j = 1,2, \ldots, 21$ realizations of this data as follows:


\[
    X_{F_{j}}= 
\begin{cases}
     X_{FA} & \text{if } j \ \text{mod 2 } == 1\\
     X_{FB}  & \text{if } j \ \text{mod 2 } == 0
\end{cases}
\]

 Finally the dataset used for clustering $Z_{ij}$ where $i = 1, 2, \ldots, 1000$ and $j = 1,2, \ldots, 21$ is formed by mixing {\it backgrounds} $X_{B}$ and {\it foregrounds} $X_{F}$ as follows:

\[
    Z_{ij}= 
\begin{cases}
     \sum_{i=1}^{50} X_{F_{ij}} + \sum_{i=51}^{1000} X_{B_{ij}} & \text{if } j \ \text{mod 2 } == 1\\
     \sum_{i=1}^{25} X_{F_{ij}} + \sum_{i=26}^{1000} X_{B_{ij}} & \text{if } j \ \text{mod 2 } == 0
 \end{cases}
\]

 The dataset $Z_{ij}$ formed in this manner consists of \tb{two} types of subregions ({\it foregrounds}) which define the \tb{two} classes used for k-medoids clustering. We perform 50 random splits of the dataset $Z_{ij}$ where each split consists of a training set of size 10 and a test set of size 11. The results for clustering are shown for the 4 distance measures in Table \ref{tab.ts_1d}.

\begin{table}
\centering
\caption{Clustering Performance for Univariate Time Series dataset with $\tau=1$}
\label{tab.ts_1d}
\begin{tabular}{|w|w|w|w|}
\hline
\tb{Distance \ Measure}  & Total Number of Errors & Percent Error \\
\hline
{\bf kdiff} & {0} &{0} \\
\hline
{mmd} & {219} &{39.8} \\
\hline
{{\bf MPdist}} & {0} &{0} \\
\hline
{dtwd} & {227} &{41.2} \\
\hline
\end{tabular}
\end{table}

 From the results it can be seen that both {\bf kdiff} and {\bf MPdist} produce the best clustering performance with $0$ errors for this dataset. This is attributed to the fact that the subregions of interest are well defined for both classes and using suitable values of parameters determined from training it is possible to accurately cluster all the time series data into \tb{two} separate groups. 
\tb{On the other hand the performance of mmd is inferior to both {\bf kdiff} and {\bf MPdist} because the {\it backgrounds} are well separated with different mean values for time series within and across the two classes.}
This results in time series \tb{even} belonging to the same class to be placed in separate clusters when mmd is used as a \tb{distance measure}. 
\tb{Similarly dtwd suffers from poor performance as this distance measure tends to \tb{place} time series with smaller separation between the mean background values in the same cluster.
However these may have distinct values for the {\it foregrounds} i.e. they can in general belong to different classes and as a result this causes errors during clustering.} 

 \tb{ {\bf \it Noise robustness} \ We explore the performance of the distance measures by considering noisy foregrounds. For foreground $X_{FA}$ data $Y_i$ for $i = 1,2,\ldots, 50$ are simulated using the model (\ref{eq.ts_1d}). The series $W_i$ is constructed via an AR(1) model driven by i.i.d errors $\sim N(0, 100)$. The AR(1) coefficient is set to $0.1$ and $\mu=10$. For foreground $X_{FB}$ data $Y_i$ for $i = 1,2,\ldots, 25$ are simulated using the model (\ref{eq.ts_1d}). The series $W_i$ is constructed via an AR(1) model driven by i.i.d errors $\sim N(0, 1)$. The AR(1) coefficient is set to $0.1$ and $\mu=-10$. Following this the foreground datasets $X_{F_{j}}$ and the dataset used for clustering $Z_{ij}$ where $i = 1,2,\ldots,1000$ and $j=1,2,\ldots,21$ are formed in the same manner as described earlier. We show example time series realizations for $\tau= 1$ and $10$ in Figures \ref{fig.ar_sd_1} and \ref{fig.ar_sd_10} respectively. Each figure contains plots of two time series with mean $=-10,10$ as per the construction of foreground $X_{FB}$ for the original and noisy case and show the relative separation between the realizations. }

\begin{figure}[!tbp]
  \centering
  \begin{minipage}[b]{0.4\textwidth}
    \includegraphics[width=2.8 in, height=2.5 in]{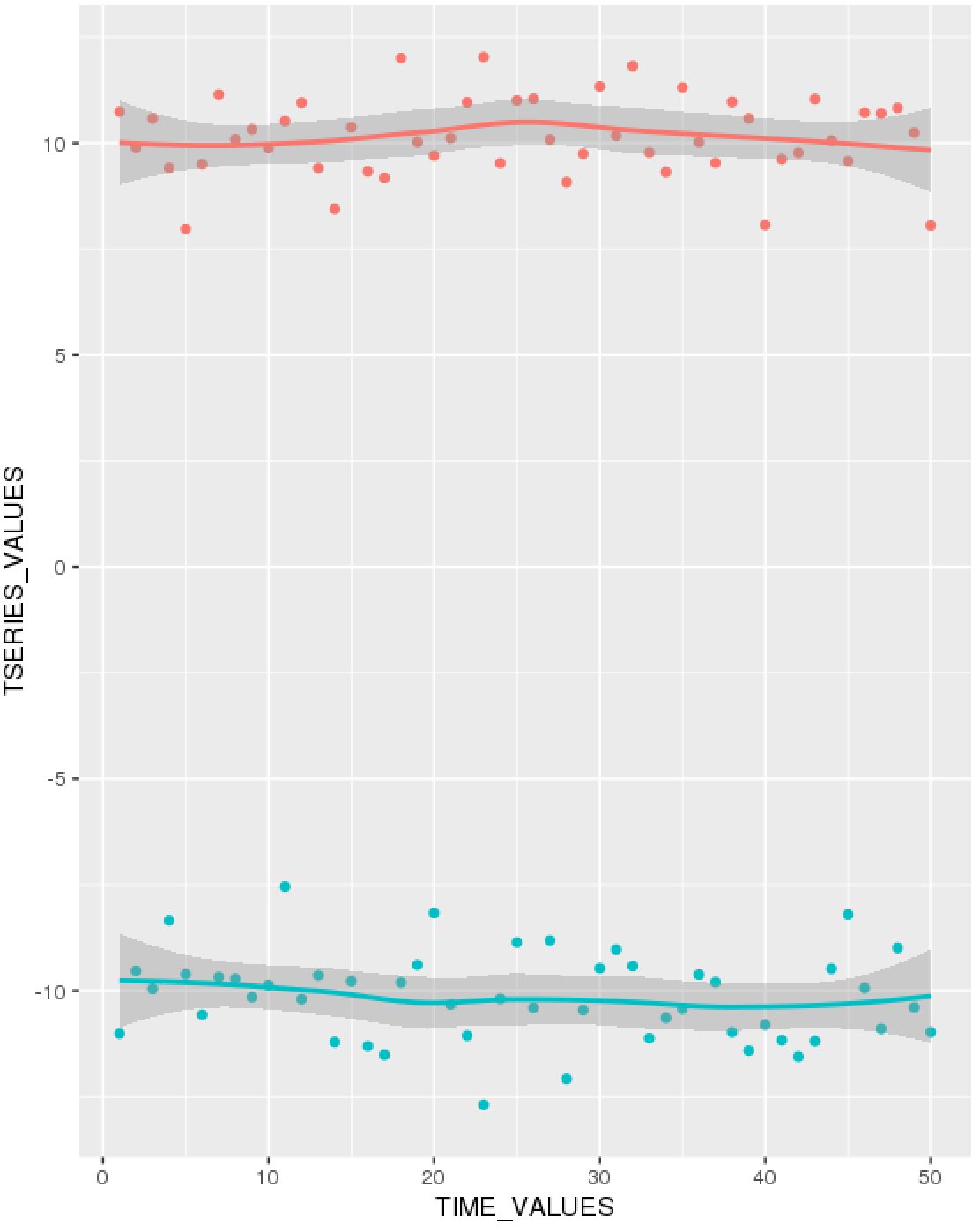}
    \caption{Foreground Time Series Realization with mean $=10,-10$ for $\tau=1$}
    \label{fig.ar_sd_1}
  \end{minipage}
  \hfill
  \begin{minipage}[b]{0.4\textwidth}
    \includegraphics[width=2.8 in, height=2.5 in]{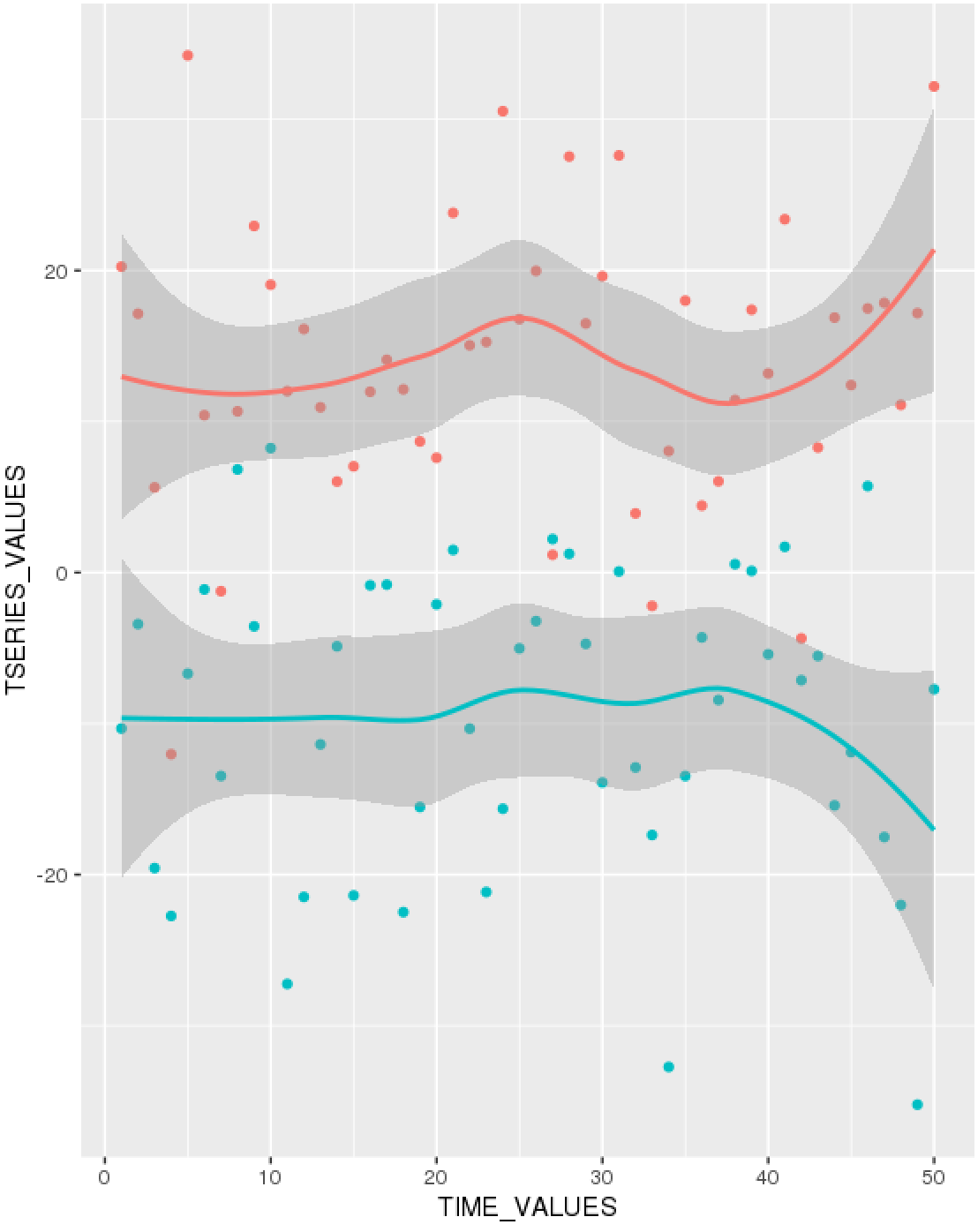}
    \caption{Foreground Time Series Realizations with mean $=10,-10$ for $\tau=10$}
    \label{fig.ar_sd_10}
  \end{minipage}
\end{figure}


 \tb{ The results for clustering using these noisy foregrounds are shown in Table \ref{tab.ts_1d_noise}. The data shows empirically that as the noise level of the foreground increases {\bf kdiff} is more resilient and performs better than {\bf MPdist}. This is because after constructing sliding window based embeddings over the original data, {\bf MPdist} is computed using Euclidean metric based cross-similarities between the embeddings whereas {\bf kdiff} is estimated using kernel based self and cross similarities over the embeddings.}

\begin{table}
\centering
\caption{Clustering Performance for Univariate Time Series dataset with $\tau=10$}
\label{tab.ts_1d_noise}
\begin{tabular}{|w|w|w|w|}
\hline
\tb{Distance \ Measure}  & Total Number of Errors & Percent Error \\
\hline
{\bf kdiff} & {20} &{3.6} \\
\hline
{mmd} & {219} &{39.8} \\
\hline
{{\bf MPdist}} & {64} &{11.6} \\
\hline
{dtwd} & {220} &{40.0} \\
\hline
\end{tabular}
\end{table}

\subsection{Simulation: Matching sub-regions in 3-dimensional Time Series in Spherical Coordinates}
\label{kdiff.sim.multiv_ts}

We generate a 3d multivariate {\it background} dataset $s_B$ as follows. Data $Y_{a_{i}}$ for $i = 1,2,\ldots, 1000$ are simulated using the model (\ref{eq.ts_1d}). \tb{To generate this} the series $W_i$ is constructed via an AR(1) model driven by i.i.d errors $\sim N(0, 1)$. The AR(1) coefficient is set to $0.1$. Similarly data $Y_{b_{i}}$ for $i = 1,2,\ldots, 1000$ are simulated using the model (\ref{eq.ts_1d}). \tb{To generate this} the series $W_i$ is constructed via an AR(1) model driven by i.i.d errors $\sim N(0, 1)$. The AR(1) coefficient is set to $0.1$. 

 \tb{Following the generation of $W_{i}$ values for the data ${\bf Y} = (Y_a, Y_b)$} we form a {\it background} dataset $X_{B_{j}}$ in this 2d space by generating $j = 1,2, \ldots, 21$ realizations of this data ${\bf Y}$ 
where the mean $\mu$ for realization $j$ of each pair is set as below:

\[
    \mu= 
\begin{cases}
    100*j & \text{if } j \geq 1 \ \text{and } \  j \leq 10\\
    100*(10-j) & \text{if } j \geq 11 \ \text{and } \  j \leq 20\\
    0              & \text{j=21}
\end{cases}
\]

 \tb{Our next step involves transforming \tb{these} 21 instances of the 2d backgrounds into a 3d spherical surface of radius 1 as described in the following steps. We first map each series $Y_a$ and $Y_b$ linearly into the region $[0, \pi/2]$.  The corresponding mapped series are denoted as $Y_{a_{s}}$ and $Y_{b_{s}}$ respectively.  To ensure that the {\it backgrounds} are \tb{clearly} separated we divide the region $[0, \pi/2]$ into 21 \tb{nonoverlapping} partitions for this linear mapping. The final background dataset $s_B=\{s_a, s_b, s_c\}$ is derived using Equation (\ref{eq.ts_3d}):}

%

\begin{equation}
\begin{aligned}
s_a &= \sin(Y_{b_{s}})*\cos(Y_{a_{s}})\\
s_b &= \sin(Y_{b_{s}})*\sin(Y_{a_{s}})\\
s_c &= \cos(Y_{b_{s}})
\end{aligned}
\label{eq.ts_3d}
\end{equation}


 Next we generate a 3d foreground dataset $s_F$ consisting of 2 {\it foregrounds} $s_{FA}$ and $s_{FB}$ which will enable forming the 2 classes to be considered for k-medoids clustering as follows. Data $Y_{a_{i}}$ for $i = 1,2,\ldots, 50$ are simulated using the model (\ref{eq.ts_1d}). \tb{To generate this} the series $W_i$ is constructed via an AR(1) model driven by i.i.d errors $\sim N(0, 1)$. The AR(1) coefficient is set to $0.1$ and $\mu = 10$. Similarly data $Y_{b_{i}}$ for $i = 1,2,\ldots, 50$ are simulated using the model (\ref{eq.ts_1d}). \tb{To generate this} the series $W_i$ is constructed via an AR(1) model driven by i.i.d errors $\sim N(0, 1)$. The AR(1) coefficient is set to $0.1$ and $\mu = 10$. we linearly map the original 2d data $(Y_a, Y_b)$ into the region $[\pi/2, 5\pi/8]$ as $(Y_{a_{s}}, Y_{b_{s}})$  and then perform the mapping as given in Equation (\ref{eq.ts_3d}) to form the {\it foreground} $s_{FA}$. The foreground $s_{FB}$ is generated in a similar manner except that $\mu = -10$ and the 2d series is linearly mapped to the region $[3\pi/4, 7\pi/8]$. We form the {\it foreground} dataset $s_{F_{j}}$ by generating $j = 1,2, \ldots, 21$ realizations of this data as follows:

\[
    s_{F_{j}}= 
\begin{cases}
     s_{FA} & \text{if } j \ \text{mod 2 } == 1\\
     s_{FB}  & \text{if } j \ \text{mod 2 } == 0
\end{cases}
\]

 Finally the dataset used for clustering $Z_{ij}$ where $i = 1, 2, \ldots, 1000$ and $j = 1,2, \ldots, 21$ is formed by mixing {\it backgrounds} $s_{B}$ and {\it foregrounds} $s_{F}$ as follows:

\[
    Z_{ij}= 
\begin{cases}
     \sum_{i=1}^{50} s_{F_{ij}} + \sum_{i=51}^{1000} s_{B_{ij}} & \text{if } j \ \text{mod 2 } == 1\\
     \sum_{i=1}^{25} s_{F_{ij}} + \sum_{i=26}^{1000} s_{B_{ij}} & \text{if } j \ \text{mod 2 } == 0
 \end{cases}
\]

 The dataset $Z_{ij}$ formed in this manner consists of \tb{two} types of subregions ({\it foregrounds}) which define the \tb{two} classes used for k-medoids clustering. An illustration of the data on such a spherical surface with $5$ {\it backgrounds} and $2$ {\it foregrounds} is shown in Figure \ref{fig.spherical_5}.

{\begin{figure}[h]
 \centering
\includegraphics[width=2.8 in, height=2.5 in]{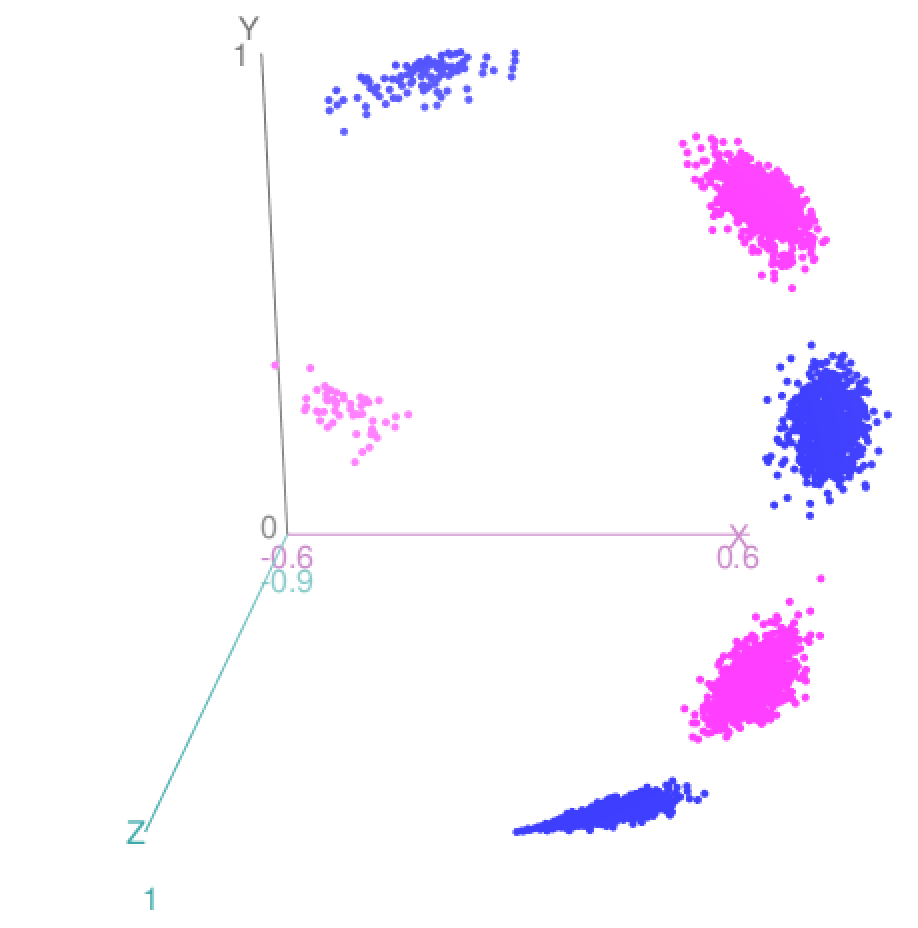}
\caption{Illustration of data consisting of 5 backgrounds and 2 foregrounds on a spherical surface, similar colors indicate association of foregrounds with respective backgrounds}
\label{fig.spherical_5}
\end{figure}}

 We perform 50 random splits of the dataset $Z_{ij}$ where each split consists of a training set of size 10 and a test set of size 11. The results for clustering are shown for the 4 distance measures in Table \ref{tab.ts_3d}.

\begin{table}[h]
\centering
\caption{Clustering Performance for 3d Multivariate Time Series dataset}
\label{tab.ts_3d}
\begin{tabular}{|w|w|w|w|}
\hline
\tb{Distance \ Measure}  & Total Number of Errors & Percent Error \\
\hline
{\bf kdiff} & {0} &{0} \\
\hline
{mmd} & {225} &{40.9} \\
\hline
{{\bf MPdist}} & {141} &{25.6} \\
\hline
{dtwd} & {227} &{41.3} \\
\hline
\end{tabular}
\end{table}

 \tb{From the results it can be seen that {\bf kdiff} produces the best clustering performance with $0$ errors for this dataset. This is attributed to the fact that the subregions of interest are well defined for both classes and using suitable values of parameters determined from training it is possible to accurately cluster all the time series data into \tb{two} separate groups. On the other hand the performance of mmd is inferior to {\bf kdiff}  because the backgrounds are well separated with different mean values for time series within and across the two classes. This results in time series even belonging to the same class to be placed in separate clusters when mmd is used as a \tb{distance measure}. Similarly dtwd suffers from poor performance as this distance measure tends to put time series with smaller separation between the mean background values in the same cluster. However these may have distinct values for the foregrounds i.e. they can in general belong to different classes and as a result this causes errors during clustering. For this dataset the performance of {\bf MPdist} is inferior to {\bf kdiff} even though the former can find matching sub-regions with zero errors in the case of univariate time series. This difference is attributed to the nature of the spherical region over which the sub-region matching is done where the 1-nearest neighbor strategy employed by {\bf MPdist} using Euclidean metrics to construct the distance distribution. In case of spherical surfaces it is necessary to use appropriate geodesic distances for nearest neighbor search as discussed in (\cite{lunga2011spherical}). This issue is resolved in {\bf kdiff} which can find the matching subregion over a non Euclidean region which in this case is a spherical surface thereby giving the most accurate clustering results for this dataset.}

\subsection{Real life example: MNIST-M dataset}
\label{kdiff.real.mnistm}

\tb{The MNIST-M dataset used in \cite{ganin2016domain, haeusser2017associative} was selected as a real-life example to demonstrate the differences in clustering performance using the \tb{four distance measures} {\bf kdiff}, mmd, {\bf MPdist} and dtwd. The MNIST-M dataset consists of MNIST digits \cite{lecun1998gradient} which are difference blended over patches selected from the BSDS500 database of color photos \cite{arbelaez2010contour}. In our experiments where we consider k-medoid clustering over $k=2$ classes we select 10 instances each of the MNIST digits 0 and 1 to be blended with a selection of background images to form our dataset MNIST-M-1. Since BSDS500 is a dataset of color images the components of this dataset are random fields whose dimensions are $28$\ x\ $28$\ x\ $3$. We form our final dataset for clustering consisting of random fields with dimensions $28$\ x$\ 28$ by averaging over all three channels. Examples of individual zero and one digits on different backgrounds for all three channels of MNIST-M-1 are shown in Figures \ref{fig_0_eg1}, \ref{fig_0_eg2}, \ref{fig_1_eg1} and \ref{fig_1_eg2}.}

%

\begin{figure}[H]
  \centering
  \begin{minipage}[b]{0.4\textwidth}
    \includegraphics[width=2.1 in, height=1.8 in]{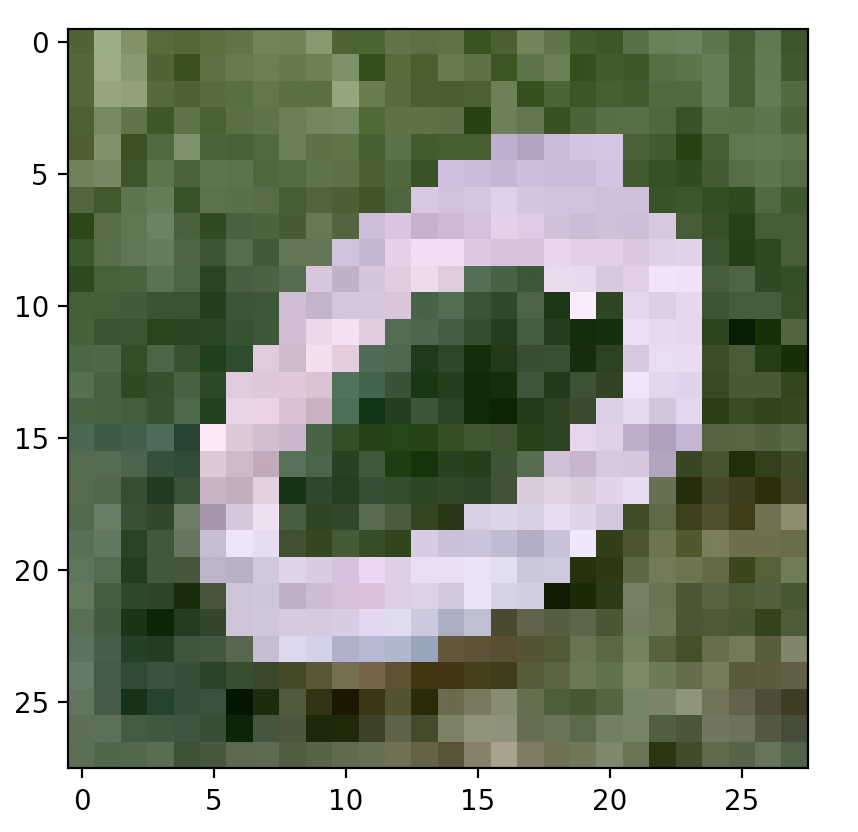}
    \caption{Example 1 of MNIST-M-1 digit zero}
    \label{fig_0_eg1}
  \end{minipage}
  \hfill
  \begin{minipage}[b]{0.4\textwidth}
    \includegraphics[width=2.1 in, height=1.8 in]{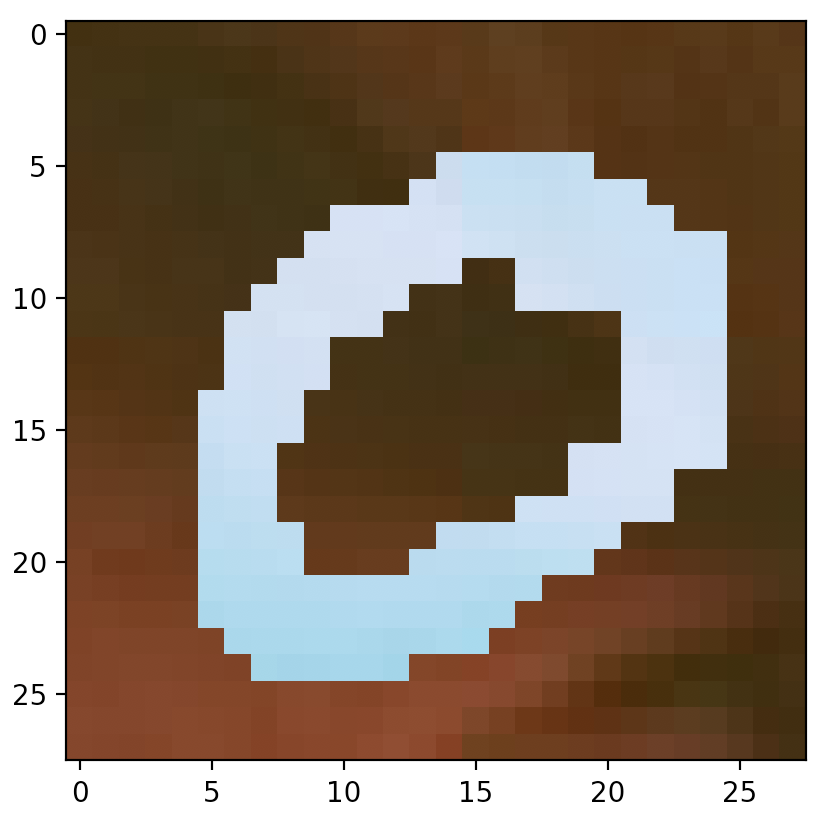}
    \caption{Example 2 of MNIST-M-1 digit zero}
    \label{fig_0_eg2}
  \end{minipage}
\end{figure}

%

\begin{figure}[H]
  \centering
  \begin{minipage}[b]{0.4\textwidth}
    \includegraphics[width=2.1 in, height=1.8 in]{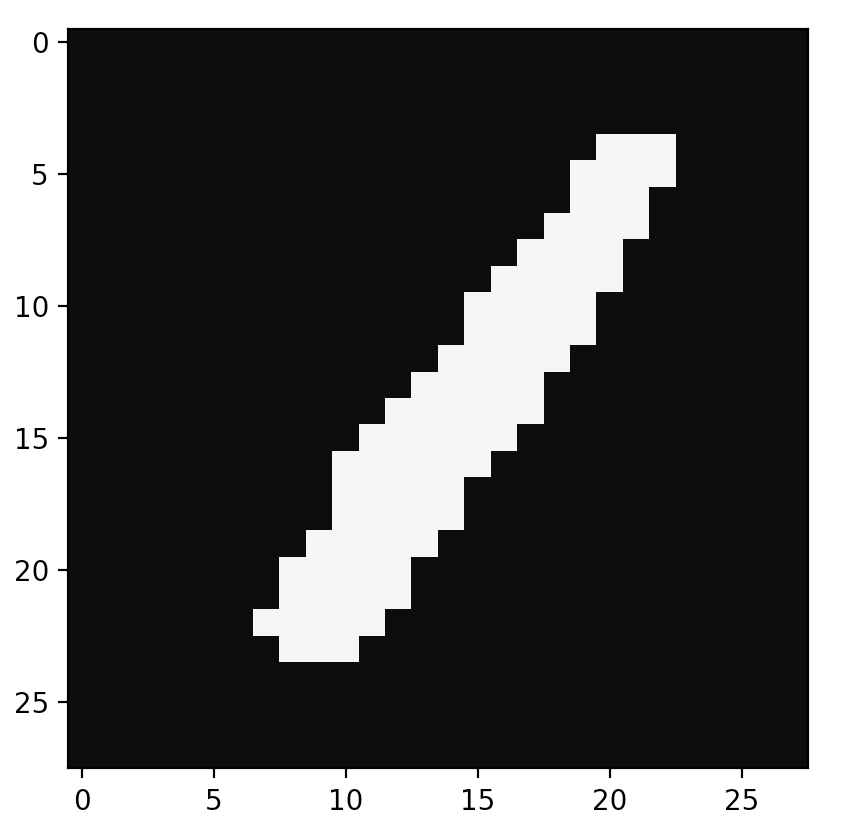}
    \caption{Example 1 of MNIST-M-1 digit one}
    \label{fig_1_eg1}
  \end{minipage}
  \hfill
  \begin{minipage}[b]{0.4\textwidth}
    \includegraphics[width=2.1 in, height=1.8 in]{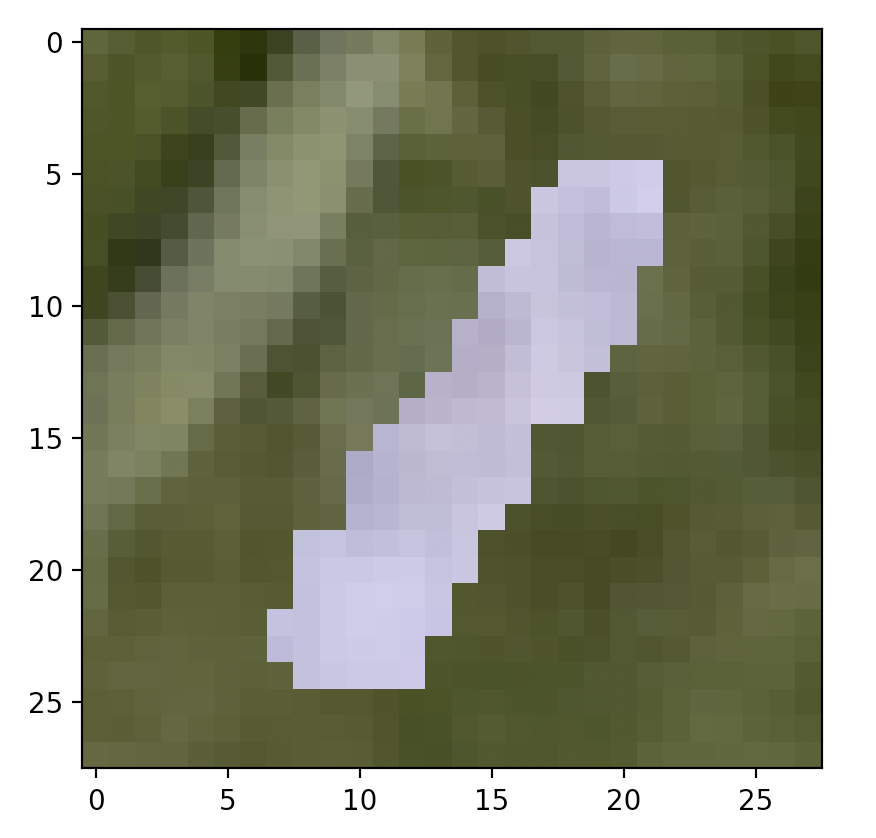}
    \caption{Example 2 of MNIST-M-1 digit one}
    \label{fig_1_eg2}
  \end{minipage}
\end{figure}

 We perform 50 random splits of the dataset where each split consists of a training set of size 10 and a test set of size 10.  The results for clustering are shown for the distance measures in Table \ref{tab.mnistm}.

\begin{table}[h]
\centering
\caption{Clustering Performance for MNIST-M-1}
\label{tab.mnistm}
\begin{tabular}{|w|w|w|w|}
\hline
\tb{Distance \ Measure}  & Total Number of Errors & Percent Error \\
\hline
{\bf kdiff} & {91} &{18.2} \\
\hline
{mmd} & {131} &{26.2} \\
\hline
{{\bf MPdist}} & {68} &{13.6} \\
\hline
{dtwd} & {149} &{29.8} \\
\hline
\end{tabular}
\end{table}

 \tb{From the results it can be seen that for MNIST-M-1 {\bf MPdist} somewhat outperforms our proposed \tb{distance measure} {\bf kdiff} however the latter is superior to both mmd and dtwd. Since in general the background statistics of the MNIST-M images are different, two images belonging to the same class can be placed in separate clusters when mmd is used as a \tb{distance measure} and this causes mmd to underperform versus {\bf kdiff}. Similarly dtwd suffers from poor performance as this distance measure tends to put images with smaller separation between the mean background values in the same cluster. However these may have distinct values for the foregrounds i.e. they can in general belong to different classes and as a result this causes errors during clustering.}




 \tb{ {\bf \it Noise robustness} \ Following the discussion in Section \ref{kdiff.sim.univ_ts} we explore the performance of the distance measures by considering a selection of noisy backgrounds from the BSDS500 database over which the same 10 instances of the MNIST digits 0 and 1 are blended to form a second version of our dataset called MNIST-M-2. Similar to the earlier case we form our final dataset for clustering consisting of random fields with dimensions 28 x 28 by averaging over all three channels of the color image. Examples of individual zero and one digits on different backgrounds for a single channel are shown in Figures \ref{fig_0_n_eg1}, \ref{fig_0_n_eg2}, \ref{fig_1_n_eg1} and \ref{fig_1_n_eg2}. Note that these correspond to the same MNIST digits shown in Figures \ref{fig_0_eg1}, \ref{fig_0_eg2}, \ref{fig_1_eg1} and \ref{fig_1_eg2} however are blended with different backgrounds which have been chosen such that the distinguishability of the two classes is reduced.}

\begin{figure}[H]
  \centering
  \begin{minipage}[b]{0.4\textwidth}
    \includegraphics[width=2.1 in, height=1.8 in]{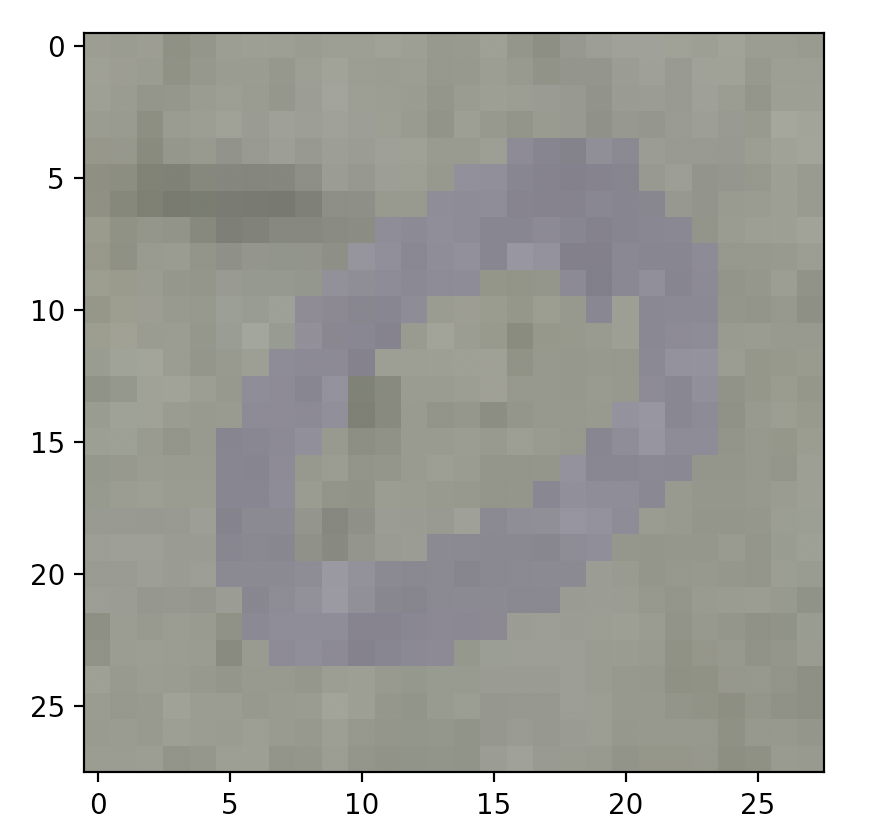}
    \caption{Example 1 of MNIST-M-2 digit zero}
    \label{fig_0_n_eg1}
  \end{minipage}
  \hfill
  \begin{minipage}[b]{0.4\textwidth}
    \includegraphics[width=2.1 in, height=1.8 in]{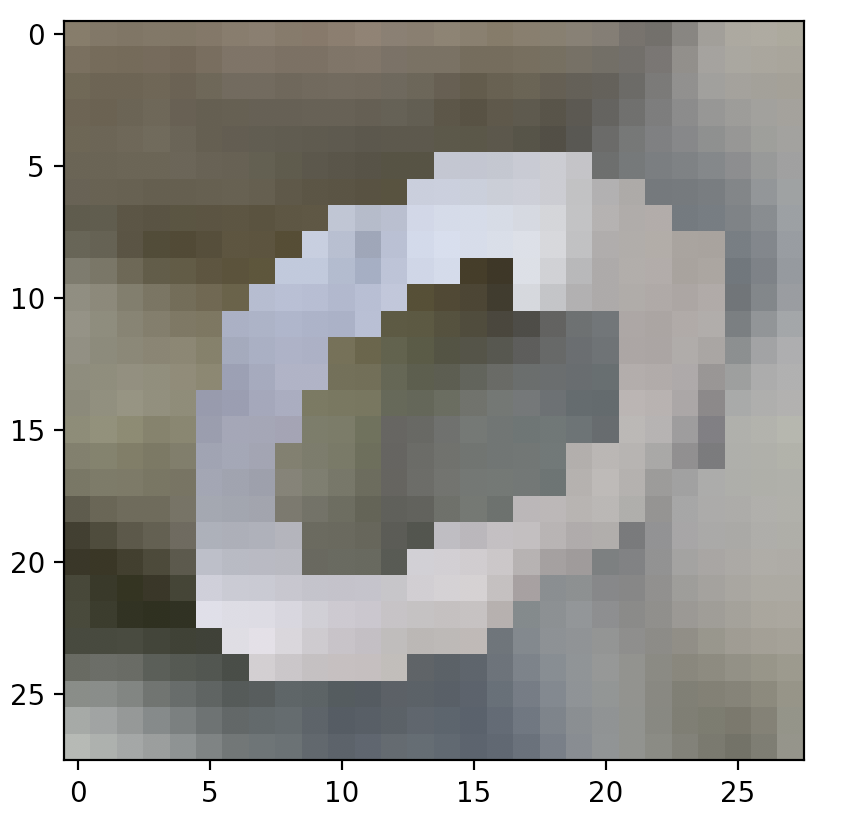}
    \caption{Example 2 of MNIST-M-2 digit zero}
    \label{fig_0_n_eg2}
  \end{minipage}
\end{figure}

%

\begin{figure}[H]
  \centering
  \begin{minipage}[b]{0.4\textwidth}
    \includegraphics[width=2.1 in, height=1.8 in]{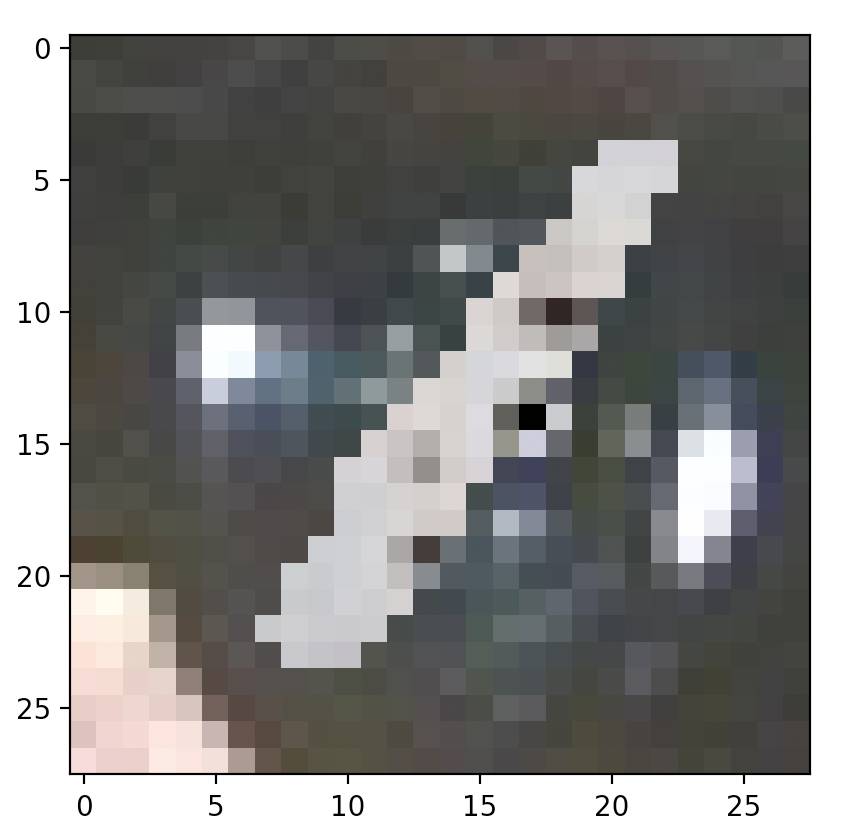}
    \caption{Example 1 of MNIST-M-2 digit one}
    \label{fig_1_n_eg1}
  \end{minipage}
  \hfill
  \begin{minipage}[b]{0.4\textwidth}
    \includegraphics[width=2.1 in, height=1.8 in]{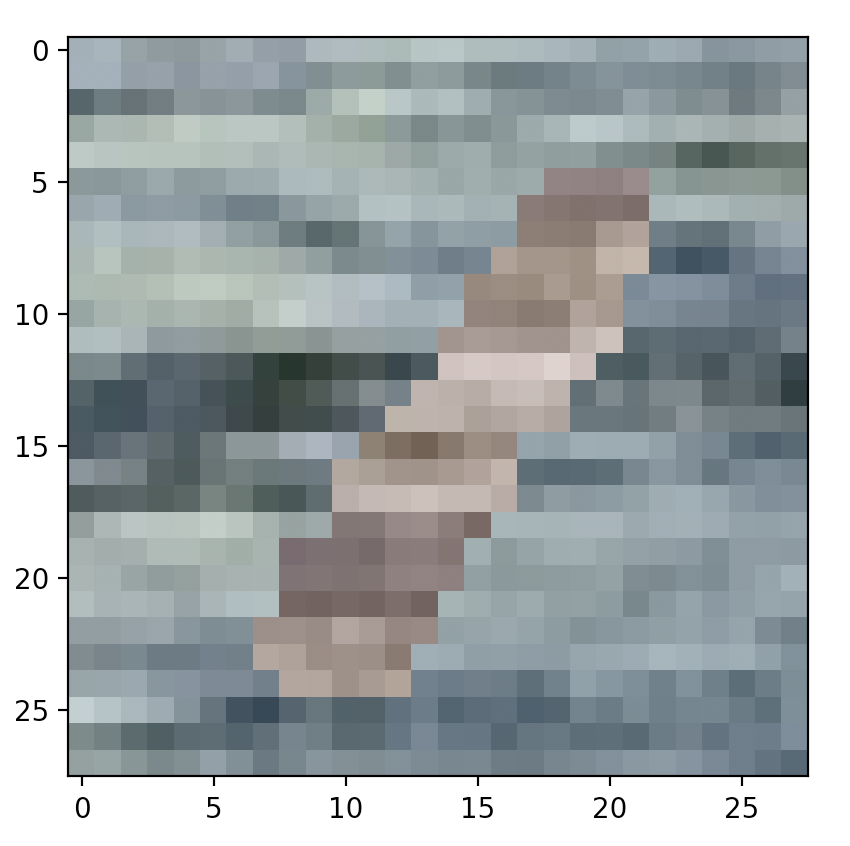}
    \caption{Example 2 of MNIST-M-2 digit one}
    \label{fig_1_n_eg2}
  \end{minipage}
\end{figure}

 \tb{ We use the Kolmogorov-Smirnov (KS) test statistic to characterize the differences between the backgrounds (BSDS500 images) and the foregrounds (MNIST digits 0 and 1) as below:}

\begin{itemize}
\item \tb{The mean KS statistic between the distribution of the pixels where a digit 0 is present and the distribution of the pixel values which make up the background ({\bf KS-bg-fg-0})}
\item \tb{The mean KS statistic between the distribution of the pixels where a digit 1 is present and the distribution of the pixel values which make up the background ({\bf KS-bg-fg-1})}
\item \tb{The mean KS statistic between pairs of distributions which make up the corresponding backgrounds ({\bf KS-bg})}
\end{itemize}

 \tb{The KS values shown in Table \ref{tab.mnistm_ks} confirm our visual intuition that the distinguishability of the foreground (MNIST 0 and 1 digits) and the background is less for MNIST-M-2 as compared to MNIST-M-1. Additionally it can be seen that for the noisier dataset MNIST-M-2 the separation between the background distribution of pixels is less than that of MNIST-M-1.}

\begin{table}[h]
\centering
\caption{KS statistic for MNIST-M foregrounds and backgrounds}
\label{tab.mnistm_ks}
\begin{tabular}{|w|w|w|w|}
\hline
{Dataset}  & {\bf KS-bg-fg-0} & {\bf KS-bg-fg-1} & {\bf KS-bg} \\
\hline
{MNIST-M-1} & {0.998} &{0.991} & {0.358} \\
\hline
{MNIST-M-2} & {0.889} &{0.754}  & {0.202} \\
\hline
\end{tabular}
\end{table}

 \tb{We perform 50 random splits of the dataset $Z$ where each split consists of a training set of size 10 and a test set of size 10.  The results for clustering are shown for the distance measures in Table \ref{tab.mnistm_noisy}.}

\begin{table}
\centering
\caption{Clustering Performance for MNIST-M}
\label{tab.mnistm_noisy}
\begin{tabular}{|w|w|w|w|}
\hline
\tb{Distance \ Measure}  & Total Number of Errors & Percent Error \\
\hline
{\bf kdiff} & {183} &{36.6} \\
\hline
{mmd} & {186} &{37.2} \\
\hline
{{\bf MPdist}} & {197} &{39.4} \\
\hline
{dtwd} & {197} &{39.4} \\
\hline
\end{tabular}
\end{table}


 \tb{From the results it can be seen that for this noisy dataset the clustering accuracy results for all four distance measures are lower as expected, however  {\bf kdiff}  slightly outperforms {\bf MPdist}. As discussed in Section \ref{kdiff.sim.univ_ts} this can be attributed to the fact that in such cases with a lower signal to noise ratio between the foreground and the background {\bf kdiff} which is estimated using kernel based self and cross similarities over the embeddings can outperform {\bf MPdist} which is computed using only Euclidean metric based cross-similarities over the embeddings. The expected noise characterizaion is confirmed by our KS statistic values of {\bf KS-bg-fg-0} and {\bf KS-bg-fg-1}  in Table \ref{tab.mnistm_noisy}. Moreover the lower values of the KS statistic value {\bf KS-bg} for MNIST-M-2 compared to MNIST-M-1 manifest in similar clustering performances of mmd and kdiff for MNIST-M-2 in contrast with the trends for MNIST-M-1.}


 \tb{{\bf \it Additional comments} \ For {\bf kdiff} we used $L = SL * SL$ windows for capturing the image sub-regions leading to $(n - SL + 1)^2$ embeddings which were subsequently "flattened" to form subsequences of size $L = SL^2$ over which {\bf kdiff} was estimated using a one dimensional Gaussian kernel. This process can be augmented by estimating {\bf kdiff} with two dimensional anisotropic Gaussian kernels to improve performance. However this augmented method of {\bf kdiff} estimation  using a higher dimensional kernel with more parameters will significantly increase the computation time and implementation complexity. Note that in the case of {\bf MPdist} flattening the subregion is not as much of an issue since it does not use kernel based estimations which need accurate bandwidths.}

\section{Conclusions and Future Work}
\label{kdiff.conclusions}

 \tb{In this work we have proposed a kernel-based measure {\bf kdiff} for estimating distances between time series, random fields and similar univariate or multivariate and possibly non-iid data. Such a distance measure can be used for clustering and classification in applications where data belonging to a given class match only partially over their region of support. In such cases {\bf kdiff} is shown to outperform both Maximum Mean Discrepancy and Dynamic Time Warping based distance measures for both synthetic and real-life datasets. We also compare the performance of {\bf kdiff}  which is constructed using kernel-based embeddings over the given data versus {\bf MPdist} which uses Euclidean distance based embeddings. In this case we empirically demonstrate that for data with high signal-to-noise ratio between the matching region and the background both {\bf kdiff} and {\bf MPdist} perform equally well for synthetic datasets and {\bf MPdist} somewhat outperforms {\bf kdiff} for real life MNIST-M data. For data where the foreground is less distinguishable versus the background {\bf kdiff} outperforms {\bf MPdist} for both synthetic and real-life datasets. Additionally for multivariate time series on a spherical manifold we show that {\bf kdiff} outperforms {\bf MPdist} because of its kernel-based construction which leads to superior performance in non Euclidean spaces. Our future work will focus on application of {\bf kdiff} for applications on spherical manifolds such as text embedding \cite{meng2019spherical} and hyperspectral imagery \cite{lunga2011spherical, lunga2011unsupervised} as well as clustering and classification applications for time series and random fields with noisy motifs and foregrounds.}


\section{Acknowledgements}
\label{kdiff.ack}

 This work was supported in part by NSF awards CNS1730158, ACI-1540112, ACI-1541349, OAC-1826967, the University of California Office of the President, and the California Institute for Telecommunications and Information Technology’s Qualcomm Institute (Calit2-QI). Thanks to CENIC for the 100Gpbs networks. \tb{The authors thank Siqiao Ruan for helpful discussions}.  SD was partially supported by an Intel Sponsored Research grant.   AC was supported by NSF awards 1819222, 2012266, Russel Sage Founation 2196, and an Intel Sponsored Research grant. HNM is supported in part by NSF grant DMS 2012355 and ARO Grant W911NF2110218.



\bibliography{srinjoy_stats}
\bibliographystyle{IEEEtran}

\end{document}